%% file: edu_main.tex
\documentclass{edm_template}

\usepackage{algorithm}
\usepackage{algpseudocode}
\usepackage{graphicx}
\usepackage[show]{chato-notes}
\usepackage{caption}
\usepackage{epsfig}
\usepackage{url}
\usepackage{amsmath}
\usepackage{url}

\usepackage{graphicx}
\usepackage{caption}
\usepackage{subcaption}
\usepackage{flushend}

\input{defines}

\begin{document}

\title{Team Formation for Scheduling Educational Material in Massive Online Classes
}

\numberofauthors{1} 
\author{\alignauthor Sanaz Bahargam, D\'ora Erd\H os, Azer Bestavros, Evimaria Terzi \\
\affaddr{Computer Science Department, Boston University, Boston MA} \\ 
\email{[bahargam, edori, best, evimaria]@cs.bu.edu}
} 
\maketitle

\begin{abstract}
Whether teaching in a classroom or a Massive Online Open Course it is crucial to present the material in a way that benefits the audience as a \emph{whole}.  We identify two important tasks to solve towards this objective;  (1.) group students so that they can maximally benefit from peer interaction  and (2.) find an optimal schedule of the educational material for each group.
 Thus, in this paper, we solve the problem of team formation 
and content scheduling for   education.
Given a time frame $d$, a set of students $\bS$ with their required need to learn different activities $\bT$ and given $k$ as the number of desired groups, we study the problem of finding $k$ group  of students. The goal is to teach students within time frame $d$ such that their potential for learning is maximized and find the best schedule for each group.  
 We show this problem to be
 NP-hard  and develop a polynomial algorithm for it. We show our algorithm to be effective both on synthetic as well as a real data set. For our experiments, we use real data on students' grades in a Computer Science department. As part of our contribution, we release a semi-synthetic dataset that mimics the properties of the real data.

\end{abstract}


\section{Introduction}\label{sec:intro}
\input{intro}

\section{Related Work}\label{sec:related}
\input{related}

\section{Preliminaries}\label{sec:problem}
\input{preliminaries}

\section{The Group Schedule Problem}\label{sec:group}
\input{group}

\section{The Cohort Selection Problem}\label{sec:partition}
\input{partition}

\section{Experiments}\label{sec:experiments}
\input{experiments}


\section{Conclusion}\label{sec:conclusion}
\input{conclusion}


\bibliographystyle{abbrv}
\bibliography{edu_main}  


\end{document}

%% file: defines.tex
\newtheorem{problem}{Problem}

\newtheorem{theorem}{Theorem}

\newcommand{\squishlist}{\begin{list}{$\bullet$}
  { \setlength{\itemsep}{0pt}
     \setlength{\parsep}{3pt}
     \setlength{\topsep}{3pt}
     \setlength{\partopsep}{0pt}
     \setlength{\leftmargin}{1.5em}
     \setlength{\labelwidth}{1em}
     \setlength{\labelsep}{0.5em} } }
\newcommand{\squishend}{
  \end{list}  }

\newcommand{\calX}{{\ensuremath{\mathcal{X}}}}
\newcommand{\calY}{{\ensuremath{\mathcal{Y}}}}
\newcommand{\calU}{{\ensuremath{\mathcal{U}}}}
\newcommand{\calS}{{\ensuremath{\mathcal{S}}}}

\newcommand{\calP}{{\ensuremath{\mathcal{P}}}}
\newcommand{\calA}{{\ensuremath{\mathcal{A}}}}

\newcommand{\ttu}{{\mathrm{\tt{u}}}}

\newcommand{\bB}{{\ensuremath{\mathbf{B}}}}
\newcommand{\bb}{{\ensuremath{\mathbf{b}}}}
\newcommand{\bm}{{\ensuremath{\mathbf{m}}}}
\newcommand{\bS}{{\ensuremath{\mathbf{S}}}}
\newcommand{\bT}{{\ensuremath{\mathbf{T}}}}

\newcommand{\benefit}{{\ensuremath{\mathbf{b}}}}
\newcommand{\singschedule}{{\ensuremath{\mathbf{R}}}}

\newcommand{\req}{{\tt req}}

\DeclareMathOperator*{\argmax}{argmax}

\newcommand{\schedule}{{\tt Schedule}}

\newcommand{\kmeans}{{\tt K\_means }}

\newcommand{\randomPartitioning}{{\tt Random Partitioning }}

\newcommand{\kbenefit}{{\tt CohPart}}
\newcommand{\kbenefitSample}{{\tt CohPart\_S}}

\newcommand{\groupschedule}{{\sc group schedule}}
\newcommand{\gs}{{\sc gs}}

\newcommand{\Lgroupschedule}{{\sc L\_group schedule}}
\newcommand{\lgs}{{\sc lgs}}

\newcommand{\maxbenefit}{{\sc Cohort Selection}}

\newcommand{\cs}{{\sc catalog segmentation}}

\newcommand{\spara}[1]{\smallskip\noindent{\bf{#1}}}
\newcommand{\mpara}[1]{\medskip\noindent{\bf{#1}}}

\newtheorem{proposition}{Proposition}

\newcommand{\groundTruth}{{\tt ground truth}}
\newcommand{\random}{{\tt uniform}}
\newcommand{\normal}{{\tt normal}}
\newcommand{\pareto}{{\tt pareto}}
\newcommand{\BUCS}{{\tt BUCS}}
\newcommand{\BUCSSynth}{{\tt BUCSSynth}}

%% file: intro.tex
Education has always been regarded as one of the most important tasks of
society. Nowadays it is viewed as one of the best means to bridge the societal
inequalities gap and to help individuals to achieve their full potential.  Accordingly, many work has been dedicated to study how individuals learn and what is
the best way to teach them (see ~\cite{hpl2000,citeulike6127382} for an overview).
We recognize two substantial conclusions that studies in this area make on how to improve students' learning outcome. First,  the use of personalized education;  by shaping the content and delivery of the lessons to the individual ability and need of each student we can enhance their performance(\cite{Novikoff12, Lin2013199,JSID:JSID444,brusilovsky2003adaptive,EDM20141352}. 
Second, grouping students; working in teams with their peers helps students to access the material from a different viewpoint as well ~\cite{Agrawal:2014,ashman2003cooperative,Slavin87,Lin2013199,EDMsegio}.

In this paper we study the problem of creating personalized educational material for teams of students by taking a computational perspective.
More specifically, we focus on two problems: the first problem is how to
identify the right schedule for a group of students, when the group is apriori formed.  The second problem is how to partition a set of students into groups and design personalized schedules per group so that the benefit of students in terms of how much they learn and absorb is maximized.

Significant amount of work has been carried out  on designing personalized educational content, such as ~\cite{lu2004personalized} in the  context of online education services and more notably on designing personalized schedules  by Novikoff et al.~\cite{Novikoff12}  which has inspired our current work. Team formation in education is another well-studied area  ~\cite{Agrawal:2014,Esposito01061973,nla.cat-vn5488506} and it has been showed that students can improve their abilities by interaction and communication with other team members ~\cite{Lazarowitz}. 

However, to the best of our knowledge we are the first to formally define and study the two problems of team formation and personalized scheduling for teams in the context of education. 
Therefore, our contribution is to present formal definition of aforementioned problems, study their computational complexity and design algorithms for solving them. In addition to this we also apply our algorithms to a real dataset obtained from real students. We make our semi-synthetic dataset \BUCSSynth,  generated to faithfully mimic the real student data available on our website.

\spara{Roadmap:} The rest of the paper is organized as follows: After reviewing the related work in Section ~\ref{sec:related}, 
we define our framework and settings in Section ~\ref{sec:problem}.  In Section ~\ref{sec:group} we define group schedule problem. 
In Section ~\ref{sec:partition} we formally define \maxbenefit\   and also show its computational complexity. In the same section, we present our  \kbenefit\ to solve \maxbenefit\ . In Section  ~\ref{sec:experiments} we show usefulness of our \kbenefit\ on synthetic and real world datasets. Finally we conclude the paper in Section ~\ref{sec:conclusion}.

%% file: related.tex
Our problem is related to psychology, education and computer science including ability grouping, repetition in learning and team formation.  We review some of these works here:

\spara{Ability grouping:} Majority of the studies in this area find that over the whole population,  there definitely is a gain in academic
performance due to ability grouping ~\cite{Grossen96,Slavin87,Kulik82,Kulik92,Kerckhoff86,s1984sigma}.  Most of the studies agree, that there is high
increase to learning of students in high-ability groups. Some say there is only small
gain, while others say there is zero gain for low-ability groups. 
But even in this case, gain to the medium and high ability groups counters
these negative effects. The benefits of grouping on students' attitude has also been studied in ~\cite{Kulik82}. Authors have shown that students in grouped classes developed more positive attitudes toward the subjects they were studying than did students in ungrouped classes.



\spara{Repetition in learning:} 
Repetition has long been regarded as essential in learning. When learning a new activity for the first time, new information is gained and stored in the short-term memory. This information will be lost over time when there is no attempt to retain it ~\cite{pentland1989learning,roediger2007foundations,jaber2004variant,bruner01repetition,
agrawal2014forming,golshan2014profit,galbrun2017finding} 
 Repetition in learning and spacing effect has been even studied in computer science in ~\cite{Novikoff12}. In this work authors try to optimize a single student's learning in the light of Ebbinhaus's work. They model education process as a sequence of abstract units and these units are repeated over time.  However they did not consider the importance of having a deadline for e.g. to prepare for a test and also the fact that after enough repetitions the information will move to long-term memory and there is negligible gain from repetition. 

\spara{Team formation:} 
An earlier version of this study has appeared in \cite{bahargam2015personalized}. Team formation has been studied in operations research community ~\cite{Baykasoglu:2007,Chen1288436,Wi20099121,Zzkarian1999}, which defines the problem as finding optimal match between  people and demanded functional requirements. It is often solved using techniques such as simulated annealing, branch-and-cut or genetic algorithms ~\cite{Baykasoglu:2007,Wi20099121,Zzkarian1999}. 
It has also been studied in computer science  ~\cite{Agrawal:2014,Anagnostopoulos:2010,Kargar:2011,Lappas:2009, Majumder:2012,Rangapuram:2013,Anagnostopoulos2012}
Majority of these work focus on team formation to complete a task and minimize the communication cost among team members. The focus of these studies is on finding only one team to  perform a given task. ~\cite{Agrawal:2014} considers partitioning students in which each student has only one ability level for all the activities and each team has a set of leaders and followers.
The goal is to maximize the gain of students where gain is defined as  the number of students who can do better by interacting
with the higher ability students. Our problem differs as we consider different ability levels for different activities.

%% file: preliminaries.tex
Already Aristotle said that "it is frequent repetition that produces a natural
tendency." 
The fundamental basis of our work is the realization 
 that repetition is an essential part of learning;
 engaging with a topic multiple times \footnote{For e.g. learning about a topic multiple
 times, reiterating it, possibly in
different formats or from different viewpoints} deepens and hastens students'
engagement and understanding processes ~\cite{bruner01repetition,weibel11principles}.  
In this paper we focus on developing  optimal schedules for teaching
groups of students (e.g. classes) that observe this dependency of learning
quality on reiteration of topics.
We model a student's learning process by a sequence of topics
that she learns about. In this sequence topics may appear multiple times,
and repetitions of a topic may count with different
weights towards the overall benefit of the student. 

 Let $\bS = \{s_1,s_2,\ldots,s_n\}$ be a set of students and $\bT = \{t_1,t_2,\ldots,t_m\}$ be a set of topics. We assign topics to $d$ timeslots based on two very simple rules; only one topic can be assigned to each
timeslot but the same topic can appear in multiple slots.  A \emph{schedule} \calA\ is a collision free assignment of topics to the timeslots. \calA\ can be thought of as an ordered list of (possible multiple occurrences) of the topics.  For a topic $t \in \bT$ the tuple $\langle t,i \rangle$ denotes the $i^{th}$ occurrence of $t$ in a schedule. The notation $\calA[r] = \langle t,i\rangle$ refers to the tuple $\langle t, i \rangle$ that is assigned to timeslot $r$ in $\calA$.

\mpara{Topic requirements.} For every student $s \in \bS$ and topic $t \in \bT$ there is a number of times that $s$ has to hear about $t$ in order for $s$ to learn every aspect of this topic. We call this number the \emph{requirement} of $s$ on $t$ and denote it by the integer function $\req(s,t)$. 

\mpara{Benefits from topic.} In order for a student $s$ to be fully knowledgeable about topic $t$, he has a requirement to learn $\req(s,t)$ times about $t$. We assume that until $s$ has met his requirements,  he  gains  knowledge and hence, will benefit to some extent from every repetition of $t$. After $\req(s,t)$ repetitions of $t$, while there is no detriment, there is also no additional benefit to $s$ from hearing about $t$.
We call $\bb(s,\langle t,i\rangle)$  (Equation~\eqref{eq:benefitdef}) the \emph{benefit} of $s$ from topic $t$ when hearing about it for the $i^{th}$ time. We assume that $s$ benefits equally from each of the first $\req(s,t)$ occurrences of $t$ in \calA, thus $\bb(s,\langle t,i\rangle) = \frac{1}{\req(s,t)}$ if $i \leq \req(s,t)$. Since after this point $s$ has already mastered topic $t$, there is no additional benefit from any later repetition of $t$ and hence  $\bb(s,\langle t,i\rangle) = 0$.
 \begin{equation}\label{eq:benefitdef}
 \bb(s,\langle t,i\rangle) =
  \begin{cases}
   \frac{1}{\req(s.t)} & \text{if\ } i\leq \req(s,t) \\
   0       & \text{otherwise}
  \end{cases}
\end{equation}
Note that for ease of exposition, we assume that all repetitions of $t$ before $\req(s,t)$ carry equal benefit to $s$. However, the definition and all of our later algorithms could easily be extended to use some other function $\bb'(s,\langle t,i\rangle)$. A natural choice for example is a function, where earlier repetitions of $t$ carry higher benefit than later ones, thus $\bb'(s,\langle t,i\rangle) = \frac{1}{2^i}$. The intuition is that first you learn about the fundamentals of $t$ and later you add on additional information.

 Given the benefits $\bb(s,\langle t,i\rangle)$  there is a natural extension to define the benefit $\bB(s,\calA)$ that $s$ gains from schedule \calA. This benefit is simply a summation over all timeslots in \calA,
 \begin{equation}\label{eq:benefit}
 \bB(s,\calA) = \sum_{r =1}^d\bb(s,\calA[r])
 \end{equation}
 Remember that in Equation~\eqref{eq:benefit}, $\calA[r]$ refers to the tuple $\langle t ,i \rangle$ that is scheduled at timeslot $r$ in \calA.

Observe, that every time topic $t$ appears in the schedule \calA,
it will contribute with the same amount of benefit towards
$\bB(s,\calA)$, regardless of the exact timeslot allocation
within  \calA. 


%% file: group.tex
 In this section we investigate the problem of how to divide students in such groups and assign schedules to each group to maximize the benefit of students in every group.  

\mpara{Group benefits.} Let $P \subseteq \bS$ be a subset of the students, we refer to $P$ as a \emph{group}. The notion of the benefit of a schedule \calA\ to  a single student $s$ lends itself to a straightforward extension to the \emph{benefit of a group}.  We define the benefit $\bB(P,\calA)$ group $P$ has from \calA\ in Equation~\eqref{eq:groupbenefit} as the sum of the benefits over all students in $P$.
\begin{equation}\label{eq:groupbenefit}
\bB(P,\calA) = \sum_{s\in P}\sum_{r=1}^d \bb(s,\calA[r])
\end{equation}

\mpara{The \textbf{\groupschedule\ }problem.}
Given a group $P$, our first task is to find an optimal schedule for this group, that is  to find a schedule that maximizes the group benefit of $P$. We call this the \groupschedule\ problem (problem ~\ref{prob:groupschedule}).
\begin{problem}[\groupschedule\ ]\label{prob:groupschedule}
Let $P \subseteq \bS$ be a group of students and $\bT$ be a set of topics. For every $s \in \bS$ and $t \in \bT$ let $\req(s,t)$ be the requirement of $s$ on $t$  given for every student-topic pair. Find a schedule $\calA_P$, such that $\bB(P,\calA_P)$ is maximized for a deadline $d$.
\end{problem}

 \mpara{The \schedule\ algorithm.}
 There is a simple polynomial time algorithm that solves problem ~\ref{prob:groupschedule}. We cal this algorithm  $\schedule(P,d)$. We  present 
 $\schedule(P,d)$
in Algorithm~\ref{algo:schedule}.

Remember that for any topic $t$ the requirement $\req(s,t)$ may be  different for the different students in $P$. 
 We say that the \emph{marginal benefit}, $\bm(P, \langle t , i \rangle)$, from the $i^{th}$ repetition of $t$  (thus $\langle t, i\rangle$) to $P$ is the increase in the group benefit if $\langle t, i \rangle$ is added to \calA.  The marginal benefit of $\langle t , i \rangle$ can be computed as the sum of benefits over all students in $P$ as given in Equation~\eqref{eq:marginal}.
 \begin{equation}\label{eq:marginal}
 \bm(P, \langle t , i \rangle) = \sum_{s \in P}\bb(s,\langle t,i\rangle)
 \end{equation}
 Observe that because students have different requirements for $t$, the  subsequent repetitions of the same topic may have different (decreasing) marginal benefits. 

Algorithm~\ref{algo:schedule} is a greedy algorithm that at every timeslot chooses an instance of the  topic with the largest marginal benefit. To achieve this we maintain an array $B$
in which values  are marginal benefit of topics $t \in \bT$, if next repetition of $t$ is added to the schedule $\calA_P$.  We keep the number that topic $t$ has been added to $\calA_P$ in array $R$.

The \schedule\ algorithm is an iterative algorithm that repeats  until all $d$ timeslots in the schedule are filled; 
it selects the topic $u_t$ with the largest marginal benefit from $B$ and adds it to the schedule $\calA_P$  (Lines~\ref{ln:topic_max_mb} and~\ref{ln:add_schedule}) .  Then it updates marginal benefit of $u_t$, $B[u_t]$ (Lines~\ref{ln:update_rep}-~\ref{ln:update_marginal_benefit}).

\begin{algorithm}[ht!]
\begin{algorithmic}[1]
\Statex {\bf Input:} requirements $\req(s,t)$ for every $s \in P$ and every topic $t \in \bT$, deadline $d$.
\Statex {\bf Output:}  schedule $\calA_P$.
\State $\calA_P \leftarrow [ ]$
\State $B \leftarrow [ \bm(P, \langle t , 1 \rangle)]$ for  $t \in \bT$ \label{ln:initialbenefit}
\State $R \leftarrow [0]$ for all $t \in \bT$
\While{$|\calA_P| < d$} \label{ln:while_loop}
\State Find topic $u_t$ with maximum marginal benefit in $B$  \label{ln:topic_max_mb} 
\State $\calA_P \leftarrow \langle u_t , R[u_t] \rangle$ \label{ln:add_schedule}
\State $R[u_t]++$\label{ln:update_rep}
\State Update $B[ u_t]$ to $\bm(P, \langle t , R[ u_t] \rangle)  $  \label{ln:update_marginal_benefit}
\EndWhile
\end{algorithmic}
\caption{\label{algo:schedule} \schedule\ algorithm for computing an optimal schedule $\calA_P$ for a group $P$. }
\end{algorithm} 
 
  
  \mpara{Runtime of \schedule.} 
The runtime of Algorithm~\ref{algo:schedule} is best computed from the point of view of computing marginal benefits of topics in $B$. 
In each iteration of the loop, the marginal benefit is only recomputed for one of the topics, $u_t$ with the largest benefit which has been added to the schedule $\calA_P$ most recently. 
The total runtime of algorithm is 
$O(d(|P| + |\bT|) )$. If we keep the marginal benefits in a max-heap, we can reduce the running time to $O(d(|P| + log|\bT|) )$. 
 Algorithm~\ref{algo:schedule} yields an optimal schedule for a group $P$. 
\begin{proposition}\label{prob:scheduleisopt}
The schedule $\calA_P$ output by Algorithm~\ref{algo:schedule} yields maximal benefit $\bB(P,\calA)$ for the group $P$.
\end{proposition}

\begin{proof} Observe, that the benefit of adding the $i^{th}$ repetition $\langle t ,i \rangle$ of topic $t$ to $\calA$ is only dependent on $i$ and $t$ but not on any other topic. Hence the choice that we make in algorithm~\ref{algo:schedule} in any iteration does not change the marginal benefit $\bm(P, \langle t , i \rangle)$. Thus choosing the topic $t$ with the largest marginal benefit in any iteration of algorithm \ref{algo:schedule} results in a schedule with maximal total benefit for the group.
\end{proof}

%% file: partition.tex
So far we discussed how  to find an optimal schedule of topics for a given group of students. The next natural question is, that given a certain teaching capacity $K$ (i.e., there are $K$ teachers or $K$ classrooms available), how to divide students into $K$ groups so that each  student benefits the most possible from this arrangement.

 At a high level we  solve an instance of a  partition problem; we have to find a $K$-part partition $\calP = P_1\cup^\ast P_2 \cup^\ast \ldots \cup^\ast P_K$ of  students into groups, so that the sum of the group benefits over all groups is maximized. 
We call the problem of finding a partition that yields the highest sum of group benefits the  
 the \maxbenefit\  Problem . 
 \begin{problem}[\maxbenefit\ ]\label{prob:maxbenefit}
 Let $\bS$ be a set of students and $\bT$ be a set of topics. For every $s \in \bS$ and $t \in \bT$ let $\req(s,t)$ be the requirement of $s$ on $t$ that is  given for every student-topic pair. Find a partition \calP\ of students into $K$ groups, such that 
 \begin{equation}\label{eq:partitionP}\bB(\calP,d) = \sum_{P \in \calP} \bB(P,\calA_P)\end{equation}
 is maximized, where we assume that $\calA_P = \schedule(P,d)$ for every group.  
 \end{problem}

\begin{theorem}\label{thm:NP}
\maxbenefit\ (Problem ~\ref{prob:maxbenefit}) is NP-hard.
\end{theorem}
\begin{proof} 
We reduce the {\sc catalog segmentation} 
problem~\cite{Kleinberg04} to {\maxbenefit} problem.  \cs\ is the following problem; there
is a universe of items $\calU = \{\ttu_1, \ttu_2, \ldots, \ttu_m\}$
and subsets $\calS_1, \calS_2, \ldots, \calS_n \subseteq U$ given. Find two subsets $\calX$ and $\calY$ of $\calU$, both of size $|\calX| = |\calY| = r$, such that \begin{equation}\label{eq:catalog}
\text{CS}(\calX,\calY) = \sum_{i=1}^n\max \{|\calS_i \cap \calX|, |\calS_i \cap \calY|\}
\end{equation}
is maximized. It is proven by Kleinberg {\it et al.}~\cite{Kleinberg04} that \cs\ is NP-hard.

We now show that if we can solve {\maxbenefit} then we can also solve
 the {\cs} problem. More
specifically, we map an instance of {\cs} to an instance of  {\maxbenefit} as
 follows:
every subset $\calS_i$ in {\cs} is mapped to a student $s_i$
 in {\maxbenefit}
and element of the universe $\ttu_i\in \calU$ of {\cs} is mapped to a
 topic $t_i$ in {\maxbenefit}.
For student $s_i$ and topic $t_j$ we set the requirement
$\req(s_i,t_j) = 1$ if $\ttu_j\in \calS_i$, otherwise $\req(s_i,t_j)= nm^3$.  
We also set $d = r$ and $K= 2$. 

We can also map a solution of \maxbenefit\ to a solution of \cs\ and vice verse; let 
$\calP = \{X, Y\}$ be a partition of the students \bS\ in \maxbenefit\ and let $\calA_X$ and $\calA_Y$ be the optimal schedules for $X$ and $Y$. We define the sets \calX\ and \calY\ in \cs\ from $\calA_X$ and $\calA_Y$. Specifically, let $\{t_1^x, t_2^x, \dots, t_r^x\}$ be the topics (possible with multiplicity) that appear in $\calA_X$. Then we define $\calX = \{\ttu_{t_1^x}, \ttu_{t_2^x},\ldots,\ttu_{t_r^x} \}$ to contain the elements in $\calU$ corresponding to the topics in $\calA_X$. $\calY$ is derived in a similar way from $\calA_Y$. 

Given a solution $\calX$ and $\calY$ to \cs, we can define the partition $\calP = \{X,Y\}$ and the corresponding group schedules  $\calA_X$ and $\calA_Y$. For every $s \in \bS$ we assign $s$ to $X$ if $|\calS \cap \calX| > |\calS \cap \calY|$ and assign $s$ to $Y$ otherwise, where $\calS$ is the set in \cs\ that we identified with student $s$. Further,  the group schedule $\calA_X$ is the schedule that contains topic $t$ if and only if $\ttu_t \in \calX$. Similar, $\calY = \{t | \ttu_t \in \calY\}$. 

We show  if $\calP = \{X,Y\}$ is an optimal solution to \maxbenefit, then the corresponding solution $\calX$, $\calY$ has to be an optimal solution to \cs. 
First, observe that because of the choice of the requirements in \maxbenefit, if $\bB$ is the value of a solution to \maxbenefit, then the value of $\cs(\calX,\calY) \geq \lfloor \bB \rfloor$. Further, $\lfloor \bB(\{X,Y\},d) \rfloor = \cs(\calX,\calY)$, where $\calP = \{X,Y\}$ is derived from \calX\ and \calY.

Let us assume, that $\calP = \{X,Y\}$ is an optimal solution to \maxbenefit, but the derived \calX\ and \calY\ are not optimal for \cs. That means  there exist $\calX'$ and $\calY'$, such that $\cs(\calX,\calY) < \cs(\calX',\calY')$. However, in this case the partition  $\calP' = \{X',Y'\}$ with the schedules $\calA_{X'}$, $\calA_{Y'}$ derived from $\calX'$ and $\calY'$ would yield a higher value for \maxbenefit problem, contradicting the optimality of $\calP$.
 \end{proof}

\subsection{Partition algorithms.}\label{sec:partitionalgo} 
We first  describe  briefly two popular algorithms for clustering,  \kmeans\ and \randomPartitioning\ and how it is applied to our problem. Then we proceed to present our solution, \kbenefit\ to the \maxbenefit\ and a sampling-based speedup, \kbenefitSample\ .


\randomPartitioning\ is assigning each point randomly  to a cluster. We use this partitioning as a baseline to compare our algorithm with. Also we use it as the initialization part of our \kbenefit\ algorithm.

\kmeans\  is a clustering method used to minimize the average squared distance between points in the same cluster. Solving \kmeans\ problem ~\cite{hartigan1979algorithm} exactly is NP-hard. Lloyd's algorithm ~\cite{Lloyd82leastsquares} solves this problem by choosing $k$ centers randomly and assigning the points to the closest center. Then the centers are recomputed based on the points assigned to it. These two phases are repeated until there is no more improvement on the cost of clustering. In our setting the students are the data points and the repetition for each topic represent each dimension. 

\mpara{\kbenefit\ algorithm.} The \kbenefit\ algorithm (Cohort Partitioning)
is presented in algorithm ~\ref{algo:kbenefit} and consists of two phases; first there is an
initialization phase (Lines~\ref{ln:init_start}-~\ref{ln:init_end}), 
 in which a random clustering is executed on all of the students (Line ~\ref{ln:random_clustering}) and then for each partition $p_i$, the centers are computed (Lines~\ref{ln:centers_part_start}-~\ref{ln:centers_part_end}) 
 using algo ~\ref{algo:schedule}. When initial cluster centers are chosen, then there is an iterative phase (Lines ~\ref{ln:updatestart}-~\ref{ln:updateend}) where students get reassigned to clusters and  cluster centers are updated again. 



In our notations $\calA$ and $\singschedule$ both show the schedules (of a group of students or a single student). $\calA$ shows the vector of size $d$ consisting of topics and their repetitions $\langle t , R[t] \rangle$ for each time slot. $\singschedule$ is a vector of size $|\bT|$ and for each topic $t$, how many times it can be repeated in deadline $d$.

\begin{algorithm}[ht!]
\begin{algorithmic}[1]
\Statex {\bf Input:} requirements $\req(s,t)$ for a student $s \in P$ and every topic $t \in \bT$ and a single schedule $\singschedule$
\Statex {\bf Output:}  \benefit(s,\singschedule)  Benefit of $s$ from schedule $\singschedule$.
\State $b = 0$
\For{all topics $t \in \bT$}
\State   $b = b + \frac{min (req(s,t), \singschedule[t])}{ \singschedule[t]} $\label{ln:benefitsingle}
\EndFor 
\end{algorithmic}
\caption{\label{algo:benefit_single_point} Benefit algorithm for computing the benefit of a single student $s$ from a schedule $\singschedule$ }
\end{algorithm}

\begin{algorithm}[ht!]
\begin{algorithmic}[1]
\Statex {\bf Input:} requirement $\req(s,t)$ for every $s \in \bS$ and $t \in \bT$, number of timeslots $d$, number of groups $K$.
\Statex {\bf Output:}  partition \calP.
\State $C = {\ }$
\State $\calP = \{P_1, P_2, \ldots, P_K\}$
\State Run \randomPartitioning\ on the students and obtain $P_i$'s  \label{ln:random_clustering} \label{ln:init_start} 
\For{$i=1,\ldots,K$} \label{ln:centers_part_start}
\State $c_i = \schedule(P_i,d)$
\EndFor \label{ln:centers_part_end} \label{ln:init_end} 

\While{convergence is achieved}\label{ln:updatestart}
\For{all students $s \in \bS$}
\State $P_i \leftarrow s$, such that $i = \argmax_{j = 1,\ldots,k}\benefit(s,c_j)$\label{ln:initassign}
\EndFor\label{ln:initend}

\For{$i=1,\ldots,K$}
\State $c_i = \schedule(P_i,d)$
\EndFor

\EndWhile\label{ln:updateend}
\end{algorithmic}
\caption{\label{algo:kbenefit} \kbenefit\  for computing the
  partition \calP\ based on the benefit of students from schedules. }
\end{algorithm} 

 \mpara{Runtime : \label{algo:kbenefit}} 
\kbenefit\ is a heuristic to solve \maxbenefit problem. In each iteration of the algorithm, the group that each student can benefit the most is found and student is assign to that group. This will take $O(k |\bT|)$ for each student. Then the schedule of each group is updated and algorithm iterates until convergence is achieved. The total running time of each iteration is  $O(k|\bS||\bT|)$. In our experiments we observed that our algorithm converges really fast, less than  a few tens of iterations. 
 
 \mpara{\kbenefitSample\ algorithm.} The \kbenefitSample\ (Cohort Partitioning with Sampling,) resembles \kbenefit\, except that it performs clustering on a random sample of students  of size $n'$ and when clustering is finished assigns  the remaining students to the cluster with the maximum benefit  $\benefit(s,c_j)$. It reduces the running time to $O(k n' |\bT|)$  \label{runningtime:mainalgo}.  We set $n' = k*c$ for different values of $c$.
 
 \subsection{Constraints on Topic Order}
In real-life, most often we cannot pick any scheduling of topics we like. Instead, there are strict precedence constraints among the topics. For example, one has to learn addition before he can learn about multiplication during a math course. Therefore, we assume that along with the topics, a set of constraints is also given. The constraints can be simple ones, such as the first occurrence of topic $t_i$ has to be before topic $t_j$, or more complicated ones, topic $t_j$ can only be scheduled after at least $r_1$ repetitions of $t_{i_1}$ and $r_2$ repetitions of $t_{i_2}$. Of course, the set of constraints can also be empty, if we do not have any of them. We can easily modify algorithm  ~\ref{algo:schedule} to take into account these constraints and check for precedence constraints.  To achieve this, after line  ~\ref{ln:while_loop} we can check for precedence constraints and in line ~\ref{ln:topic_max_mb} we choose only the topics which their precedence constraints are met.

%% file: experiments.tex
The goal of these experiments is to gain an understanding of how our clustering algorithm works in terms of performance (objective function). Furthermore, we want to understand
how the deadline parameter impacts our algorithm. We used a real world dataset, semi synthetic and  synthetic datasets.  The semi synthetic dataset and the source code to generate it are available in our website.
 We first introduce Graded Response Model (GRM) briefly, then explain different datasets and finally show how well our algorithm is doing on each dataset. 


\mpara{Item Response Theory and Graded Response Model:} 
In psychometric, Item Response Theory (IRT) is a framework for designing and evaluating tests, questions and  questionnaires. In IRT models the probability of giving a  correct answer by a student to a question is determined based on the ability of student and the difficulty of the question.
For our work we used the Graded Response Model (GRM), an advanced IRT model which fits our data well and handles partial credit values. 
Using our data on grades of students for taken courses, GRM helps us to deduce ability scores for each student and difficulty scores for each course. 
Having these score parameters, then we can generate the missing grades for courses that a student did not take.  We also used GRM to obtain a model to generate a larger dataset, i.e. BUCSSynth. 

\subsection{Datasets}
This subsection describes each dataset and their attributes. 

\mpara{BUCS data:}
The original BUCS dataset consists of grades of students in CS courses at Boston University. This data was collected from Fall 2003 to Fall 2013. Each row of data looked like: FALL 2003, CS101,  U12345, U1, C+ which shows the semester year, course number, students' BU id, undergraduate/graduate year and the grade.  It consists of 9833 students. We only considered students who were taking CS330 and CS210 (required courses to obtain a major in CS) which consisted of 398 students and 41 courses. Here the courses correspond to topics. Obviously the new dataset had some missing values, not all 41 courses were taken by those 398 students.  To fill the grades for missing (student, course) pairs, we used GRM. First using GRM, we obtained the ability and difficulty parameters for all students and all courses. The abilities \footnote{\url{ http://cs-people.bu.edu/bahargam/abilities}} and difficulties' parameters are available online\footnote{\url{http://cs-people.bu.edu/bahargam/difficulties}}.  Then for each pair of (student, course) in which student $s$ did not take course $c$, we used the ability of $s$ and difficulty of $c$ to predict the grade of course $c$ for that student. After having all grades for all courses, we had to transform these grades to the number of required repetitions to learn a course. We assumed the number of required repetition to master a course (or topic) for the smartest student is 5 (base parameter). Note that throughout a semester students review the course materials to solve homework, do project and prepare for quizzes, midterm and final exams, so they review material for at least 5 times.   Thus for students who got A, we considered 5 repetitions needed to fully master the course and as the ability (and grade) drops, number of repetition goes up (step parameter). We also tried different base and step values for our experiments. 
 
\mpara{BUCSSynth data:}
In order to see how well our algorithm scales to a larger dataset, we generated a synthetic data,  based on the obtained parameters from GRM. We call this dataset BUCSSynth. From BUCS dataset, we observed that  the ability of students follows a normal distribution with 
$ \mu= 1.13$ and $\sigma = 1.41$.
Applying GRM to BUCS data, we obtained difficulty parameters for 41 courses. 
In order to obtain difficulties for 100 courses, we used the following approach:\\
\-\ \-\ \textbf{1.} Choose one of the 41 courses at random.\\
\-\ \-\ \textbf{2.} Use density estimation, smoothing and then get the CDF of the difficulties. \\
\-\ \-\ \textbf{3.} Randomly sample from the CDF to get the difficulties for a new course. \\
Using the aforementioned parameters, we generated the grades for 2000 students and 100 courses and we transformed the grades to number of repetitions similar to what we did for BUCS dataset. This dataset \footnote{\url{http://cs-people.bu.edu/bahargam/BUCSSynth}}  and the code \footnote{\url{http://cs-people.bu.edu/bahargam/BUCSSynthCode}} to generate it are  available online.

\mpara{Synthetic data:}
Our first synthetic dataset is  to generate ground truth data to compare our algorithm to ~\randomPartitioning and ~\kmeans. In this dataset we had generated 10 groups of students, each group containing 40 students. For each group we selected 5 courses and assigned repetitions randomly to those 5 courses such that the sum of repetition will be equal to the deadline\footnote{The repetition for those selected courses are not equal for the students in the same group, but for all the students in the group the sum of selected courses is equal to the deadline.}. Then for the remaining 35 courses, we filled the required number of repetitions with random numbers taken from a normal distribution with $\mu = \frac{deadline}{5}$ and $\sigma = 3$. We refer to this dataset as GroundTruth. We expect our algorithm to be able to find the right clusters of students while \kmeans\ cannot find this hidden structure.  

We have also generated the repetitions for 400 students and 40 courses using Pareto, Normal and Uniform distributions. We refer to this datastes as ~\pareto, ~\normal\  and ~\random.   To generate number of repetitions for different courses using the pareto distribution, we used the shape parameter $\alpha = 2$. For normal distribution we used $\mu = 30$ and $\sigma = 5$ and for uniform dataset we generated random numbers in the range of [5,100].

\begin{figure*} 
    \centering
    \makebox[\linewidth][c]{%
    \begin{subfigure}[b]{0.25\textwidth}
        \includegraphics[width=\textwidth]{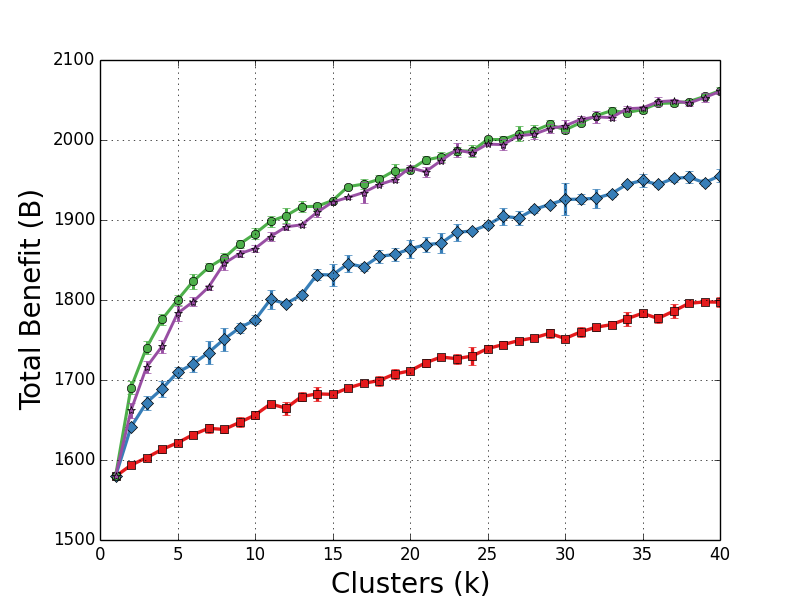}
        \caption{Ground Truth}
        \label{fig:groundtruth}
    \end{subfigure}%
    \begin{subfigure}[b]{0.25\textwidth}
        \includegraphics[width=\textwidth]{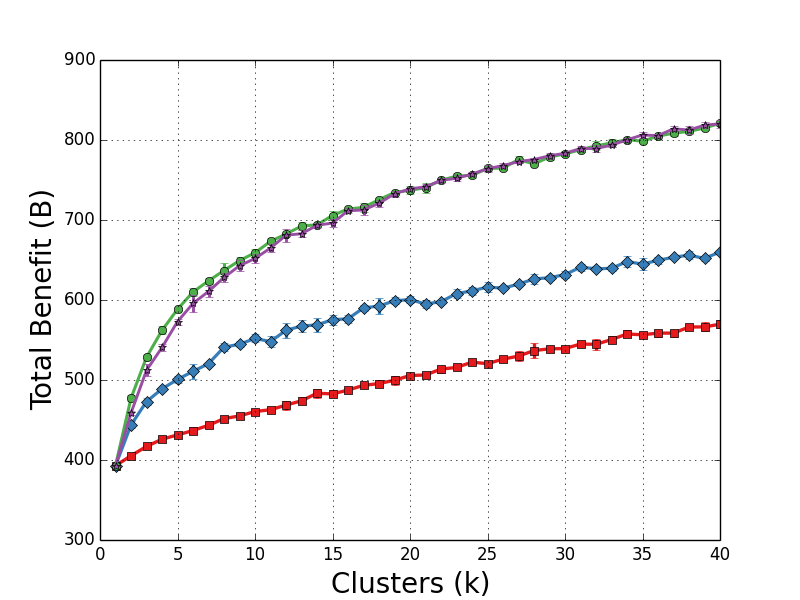}
        \caption{Random}
        \label{fig:random}
    \end{subfigure}%
    ~ 
    \begin{subfigure}[b]{0.25\textwidth}
        \includegraphics[width=\textwidth]{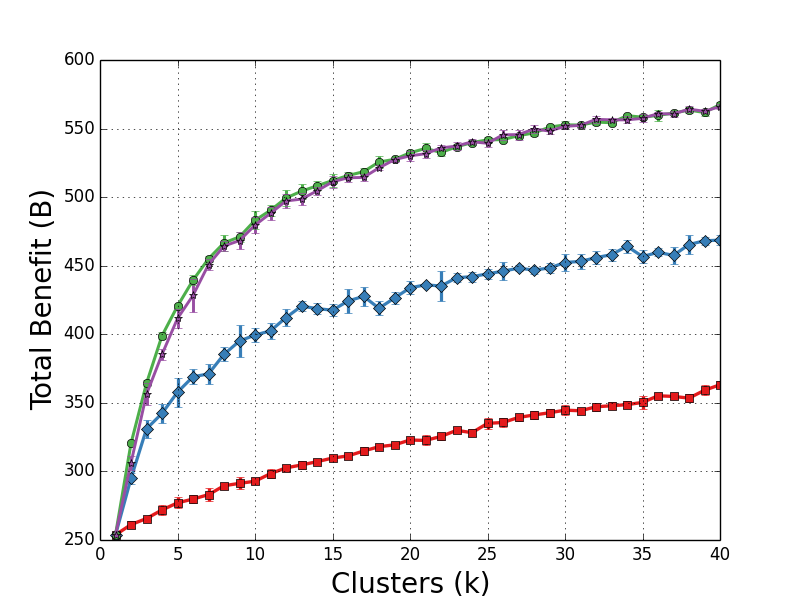}
        \caption{pareto}
        \label{fig:pareto}
    \end{subfigure}
    \begin{subfigure}[b]{0.25\textwidth}
        \includegraphics[width=\textwidth]{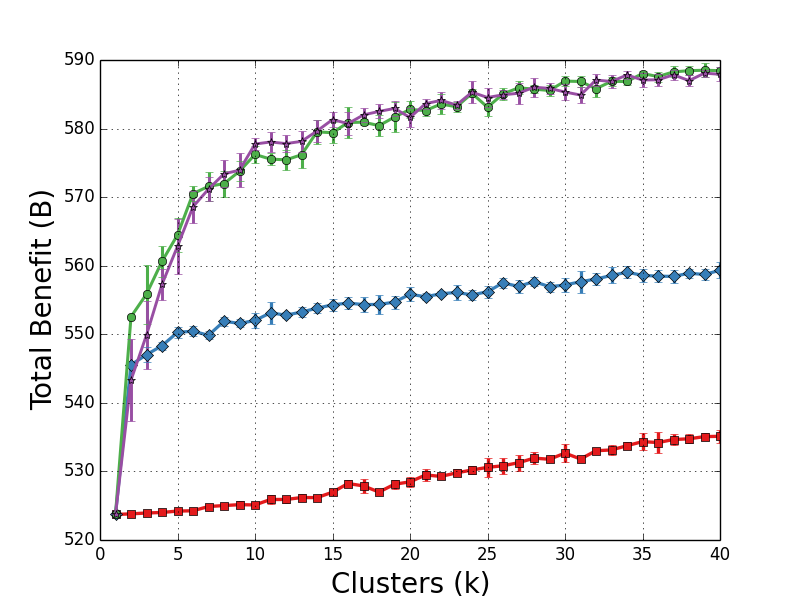}
        \caption{normal}
        \label{fig:normal}
    \end{subfigure}
    }
    \makebox[\linewidth][c]{%
    \begin{subfigure}[b]{0.25\textwidth}
        \includegraphics[width=\textwidth]{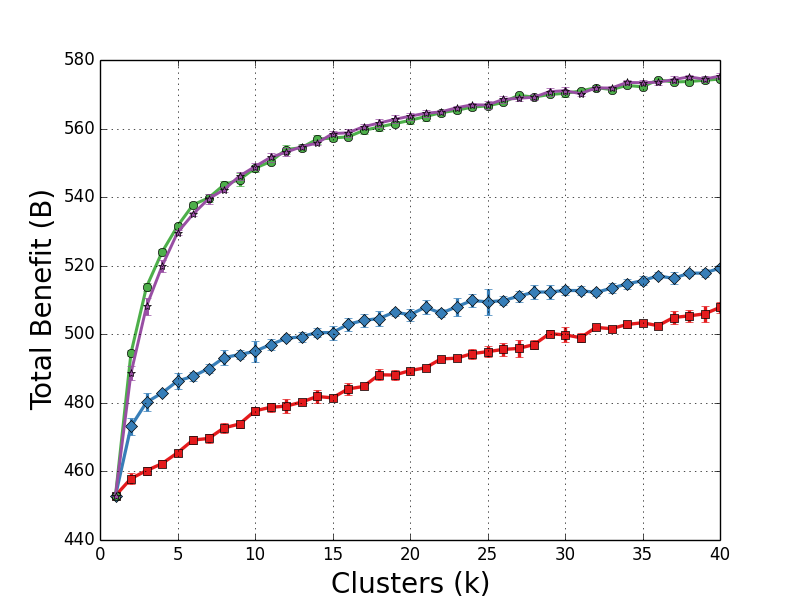}
        \caption{BUCS}
        \label{fig:BUCS}
    \end{subfigure}
    \begin{subfigure}[b]{0.25\textwidth}
        \includegraphics[width=\textwidth]{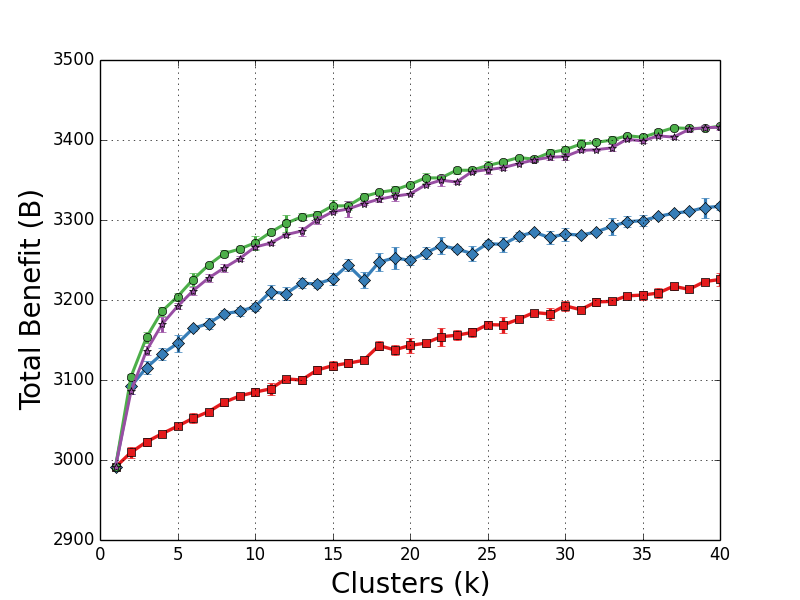}
        \caption{BUCSdeadline}
        \label{fig:BUCSdeadline}
    \end{subfigure}
    \begin{subfigure}[b]{0.25\textwidth}
        \includegraphics[width=\textwidth]{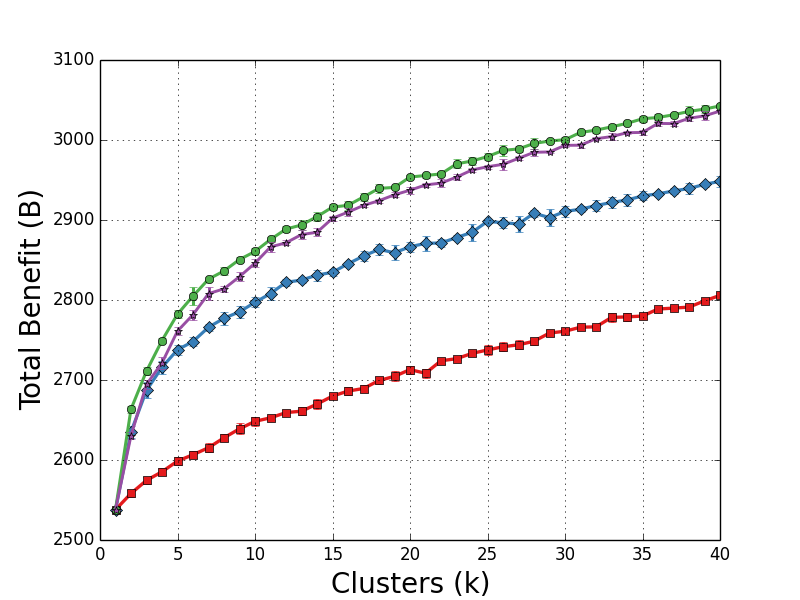}
        \caption{BUCSBase}
        \label{fig:BUCSBase}
    \end{subfigure}
    \begin{subfigure}[b]{0.25\textwidth}
        \includegraphics[width=\textwidth]{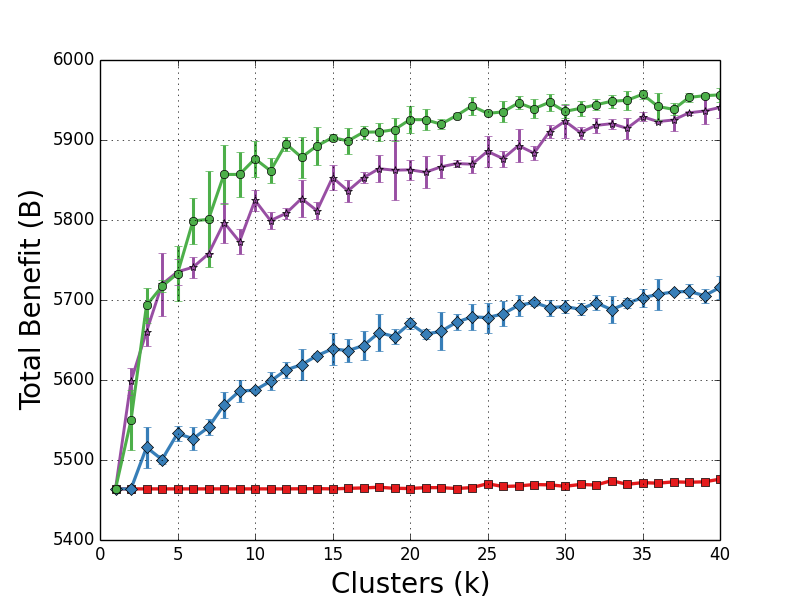}
        \caption{BUCSSynth}
        \label{fig:BUCSSynth}
    \end{subfigure}  
    }
    \caption*{\protect
\includegraphics*[width=0.5\textwidth]{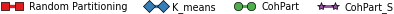}    
     }
    \caption{ Performance of  {\randomPartitioning}, {\kmeans}, {\kbenefit} and {\kbenefitSample} for {\maxbenefit} problem  on different datasets  }
\end{figure*}

\subsection{Results:}
Our experiments compare our algorithm in terms of our objective function (students' benefit) with \randomPartitioning\ and \kmeans\. Recall that the students' benefit is defines in Equation~\eqref{eq:partitionP}. 
The current algorithm is implemented in Python 2.7 and all the experiments are
run single threaded on a Macbook Air (OS-X 10.9.4, 4GB RAM). 
We compare our algorithm with \randomPartitioning and the  \kmeans algorithm, the built in k-means function in Scipy library.   Each experiment was repeated 5 times and the average results are reported in this section.  For sample size in  \kbenefitSample\ algorithm, we set parameter $c$ (explained earlier) to 4 in all experiments. 

\subsubsection{Results on Real World Datasets}
\mpara{BUCS:} 
We executed our algorithm on   \BUCS\ dataset untill reaching convergence and show how well it maximized the  benefit of learning while varying the number of clusters We  compare {\kbenefit} and {\kbenefitSample} to  \randomPartitioning\ and {~\kmeans}. The result is depicted in Figure ~\ref{fig:BUCS} where each point shows the benefit of all students when partitioning them into k groups. As we  see the \randomPartitioning\ has the lowest benefit and our algorithm has the best benefit. As the number of clusters increases (having hence fewer students in each cluster), the benefit also increases, means the schedule for those students is more personalized and closer to their individual schedule, when having one tutor for each student. The benefit grows dramatically from 1 cluster to 10 cluster. %
But after 10 cluster the increase in the potential is slower. 
  We also show the 95 confidence interval, but it was small that cannot be seen in some plots.
  
  \mpara{BUCSdeadline:}  We also show the result for different values of deadline. As the deadline increases, the gap between  \kmeans\ and our algorithm decreases. The reason is as deadline is greater we have to take into consideration more topics to teach to the students.
   Note that \kmeans\ algorithm behaves like our algorithm except it 
  considers all the courses and ignores the deadline. So the greater the deadline is, the closer  \kmeans\ gets to our algorithm. But in real life, we do not have enough time to repeat (or teach) all of the courses (for e.g. for preparation before SAT exam).  Figure ~\ref{fig:BUCSdeadline}  illustrates the case when deadline is equal to the average sum of need vectors for different students. 
   
   \mpara{BUCSBase:}
We tried different values for base and step parameters (explained earlier) and the result is depicted in Figure 
   ~\ref{fig:BUCSBase}, when the base and step are equal to 1. We observe that when the base is equal to 1 and step is also small,  \kmeans\ also performs well, but still our algorithm is doing  better than \kmeans\ . The larger is the value of base and step parameter, the better our algorithm performs.
    
\subsubsection{Results on Semi-synthetic Dataset}
\mpara{BUCSSynth dataset:} 
We  ran our algorithmon on BUCSSynth dataset  to see how well our algorithm scales for large number of students. The result is depicted in Figure ~\ref{fig:BUCSSynth}. 

\subsubsection{Results on Synthetic Datasets}
Our first set of experiments on synthetic data used the \groundTruth\ dataset.  The result is illustrated in Figure ~\ref{fig:groundtruth}. As we see \kbenefit\ and \kbenefitSample\ both are performing really well.  For all of the courses the mean required repetition is close to 10 with standard deviation 3.  We expect that students in the same group (when generating the data) should be placed in the same cluster as well after running our algorithm and the schedule should include the selected courses in each group. 
In each group students have different repetition values for the selected courses, but the sum of these selected courses is equal to the deadline and our algorithm realized this structure and only considered these selected courses to obtain the schedule. But \kmeans\ lacked this ability and did not cluster these students together. 
The next studied datasets were  {\random}, \pareto\ and \normal\ datasets  and the results are depicted in Figure ~\ref{fig:random}, ~\ref{fig:pareto} and ~\ref{fig:normal} respectively.  For these datasets also our algorithm outperformed \kmeans\ and \randomPartitioning\ .

%% file: conclusion.tex
In this paper, we highlighted the importance of team formation and scheduling educational materials for students. We suggested a novel clustering algorithm to form different teams and teach the team members based on their abilities in different topics. Our algorithm maximized the potential and benefit of team members for learning . The encouraging results that we obtained shows that our proposed solution is effective and suggest that we have  to consider personalized teaching for students and form more efficient teams.